\documentclass[a4paper,11pt,oneside]{article}

\usepackage[english]{babel}
\usepackage[T1]{fontenc}
\usepackage[utf8]{inputenc}

\usepackage[pdftex,unicode]{hyperref}
\hypersetup{pdftitle=Score-Driven Rating System for Sports}
\hypersetup{pdfauthor=Vladimír Holý and Michal Černý}

\usepackage{xcolor}
\definecolor{mycol}{rgb}{0.18,0.33,0.69}
\hypersetup{colorlinks=true, linkcolor=mycol, anchorcolor=mycol, citecolor=mycol, filecolor=mycol, urlcolor=mycol}

\usepackage[margin=60pt]{geometry}
\usepackage{amsmath}
\usepackage{amssymb}
\usepackage{bm}

\usepackage{amsthm}
\theoremstyle{definition}
 
\newtheorem{proposition}{Proposition}

\newtheorem{example}{Example}
\newtheorem*{remark}{Remark}
\newtheorem*{mainass}{Key Assumptions}

\usepackage[group-minimum-digits=3]{siunitx}

\usepackage{enumitem}

\usepackage{graphicx}
\usepackage{float}
\usepackage{hhline}
\usepackage{multirow}
\usepackage{rotating}
\usepackage{booktabs}

\usepackage[authoryear]{natbib}

\begin{document}

\begin{center}
{\Large \bfseries Score-Driven Rating System for Sports}
\end{center}

\begin{center}
{\bfseries Vladimír Holý} \\
Prague University of Economics and Business \\
Winston Churchill Square 1938/4, 130 67 Prague 3, Czechia \\
\href{mailto:vladimir.holy@vse.cz}{vladimir.holy@vse.cz}
\end{center}

\begin{center}
{\bfseries Michal Černý} \\
Prague University of Economics and Business \\
Winston Churchill Square 1938/4, 130 67 Prague 3, Czechia \\
\href{mailto:cernym@vse.cz}{cernym@vse.cz}
\end{center}

\noindent
\textbf{Abstract:}
This paper introduces a score-driven rating system, a generalization of the classical Elo rating system that employs the score, i.e.\ the gradient of the log-likelihood, as the updating mechanism for player and team ratings. The proposed framework extends beyond simple win/loss game outcomes and accommodates a wide range of game results, such as point differences, win/draw/loss outcomes, or complete rankings. Theoretical properties of the score are derived, showing that it has zero expected value, sums to zero across all players, and decreases with increasing value of a player's rating, thereby ensuring internal consistency and fairness. Furthermore, the score-driven rating system exhibits a reversion property, meaning that ratings tend to follow the underlying unobserved true skills over time. The proposed framework provides a theoretical rationale for existing dynamic models of sports performance and offers a systematic approach for constructing new ones.
\\

\\

\\

\section{Introduction}
\label{sec:intro}

The gradient of the log-likelihood, commonly referred to as the score, has a variety of uses in statistics, econometrics, and machine learning. The score plays a pivotal role in several algorithms for maximum likelihood estimation, such as the gradient ascent method and the Newton--Raphson method. It is also central to the score test, also known as the Lagrange multiplier (LM) test, which assesses constraints on parameters of models estimated by the maximum likelihood method---a concept dating back to \cite{Rao1948}. More recently, \cite{Hyvarinen2005} used score matching for the estimation of unnormalized models. Furthermore, \cite{Creal2013} and \cite{Harvey2013} utilized the score in time series models as an update term for time-varying parameters, forming the class of generalized autoregressive score (GAS) models, also known as dynamic conditional score (DCS) models, or simply score-driven models.

In this paper, we explore the use of the score in sports statistics. The term \emph{score} itself suggests a connection to sports. However, this is not the case, at least not intentionally. According to \cite{David1995}, the term \emph{normal score} was first used for the gradient of the log-likelihood by \cite{Fisher1935}, but only in the context of analyzing the genetic traits of families in which one parent carries a genetic abnormality. Later, \cite{Rao1948} used the term \emph{efficient score} generally, without a specific application in mind. Nevertheless, we demonstrate in this paper that the (statistical) score is a perfect tool for assessing outcomes of sports games, often in the form of (sports) score, thus aligning the statistical and sports terminology.

There is a recent strand of literature applying score-driven models of \cite{Creal2013} and \cite{Harvey2013} in sports statistics. \cite{Gorgi2019} used a score-driven Bradley--Terry model for tennis players, incorporating court surface effects, home advantage, and age of players. \cite{Koopman2019} presented an application to national football leagues and utilized three score-driven models, each based on a different distribution tailored to a specific variable type: the bivariate Poisson distribution for the number of goals scored and conceded, the Skellam distribution for the goal difference, and the ordered probit model for the win/draw/loss outcome. \cite{Lasek2021} used models similar to those of \cite{Koopman2019} for rating of football teams. \cite{Holy2022f} proposed a score-driven ranking model based on the Plackett--Luce distribution and illustrated its use in modeling the results of the Ice Hockey World Championships. \cite{Holy2025a} then extended the application of \cite{Holy2022f} in several ways, among others by including various predictor variables, such as physical attributes of players, their experience, and past results from related tournaments. All these papers focus on particular models and specific applications and do not study the properties of the score in the general sports context. \emph{Our first goal is to lay down the mathematical foundations for the use of the score in the statistical modeling of game results.}

Unlike most of the above-mentioned studies using score-driven models in sports, we do not focus on assessing the influence of related variables or forecasting game outcomes, but rather on fairly rating players and teams. Such ratings can be used for matchmaking---that is, pairing of players or teams and organization of tournaments. The objective of rating is conceptually very different. We do not aim to find the model with the best predictive performance, whether in-sample or out-of-sample, but rather to build a model that is fair above all. Looking back at the related studies, \cite{Koopman2019} and \cite{Holy2025a} employed mean-reverting dynamics, \cite{Gorgi2019} and \cite{Lasek2021} utilized random walk dynamics, while \cite{Holy2022f} adopted both approaches. Mean-reverting models are by design unfair for ratings---even unsportsmanlike---as they assume that future performance is determined (by the constant term) and only short-term deviations from the long-term level are possible (due to the autoregressive term). Random walk models are more suitable for ratings; however, including explanatory variables---such as an indicator for the home advantage, as done by \cite{Gorgi2019} and \cite{Holy2022f}---undermines the fairness of the rating. Among these studies, only \cite{Lasek2021} specifically dealt with the task of rating.

For games in which the results take the simple form of either a win or a loss, the Elo rating system can be applied (see \citealp{Elo1978}). The system was originally developed by the physicist Arpad Elo (1903–1992) to rate chess players. Since then, it has been applied to a wide range of sports, board games, and computer games. The main feature of the Elo rating system is that the increase (or decrease) of rating after a player's win (or loss) depends on the rating of the opponent---losing to a strong opponent does not reduce the rating by much, while winning against a strong opponent yields a significant boost. \cite{Aldous2017} noted that the Elo rating system is rather a neglected topic in the field of probability and statistics, despite offering many interesting research questions. One recent research direction is the generalization of the Elo rating system to handle richer game outcomes than a simple win or loss. For example, \cite{Szczecinski2020} generalized Elo to accommodate draws, \cite{Ingram2021} proposed a generalization for the margin of victory, \cite{Szczecinski2022} addressed outcomes in the form of ordinal categories, and \cite{Powell2023} developed a generalization for multiple players. \emph{Our second goal is to generalize the Elo rating to acommodate game outcomes in any form using the score.}

The remainder of the paper is organized as follows. In Section \ref{sec:elo}, we review the classical Elo rating system for game outcomes in the form of a win or a loss. In Section \ref{sec:score}, we propose a general rating system based on score, suitable for any game outcome that can be represented by a random variable and modeled using a probability distribution. In Section \ref{sec:equi}, we show that the Elo rating system is a special case of the score-driven rating system, provided that the logistic link function is used. In Section \ref{sec:exm}, we present examples of the score-driven rating system for specific game outcomes, such as the point difference between two players, the win/draw/loss result between two players, and the complete ranking of multiple players. In Section \ref{sec:prop}, we prove that, under suitable conditions, the score has several desirable properties, such as zero expected value, summing to zero over all players, and decreasing with the value of a player’s rating. In Section \ref{sec:dyn}, we demonstrate that the score-driven rating system possesses reversion dynamics, making it suitable as an estimate of the unobserved time-varying true skill of players. Finally, we conclude the paper in Section \ref{sec:con}.

\begin{figure}
\centering
\includegraphics[scale=0.55]{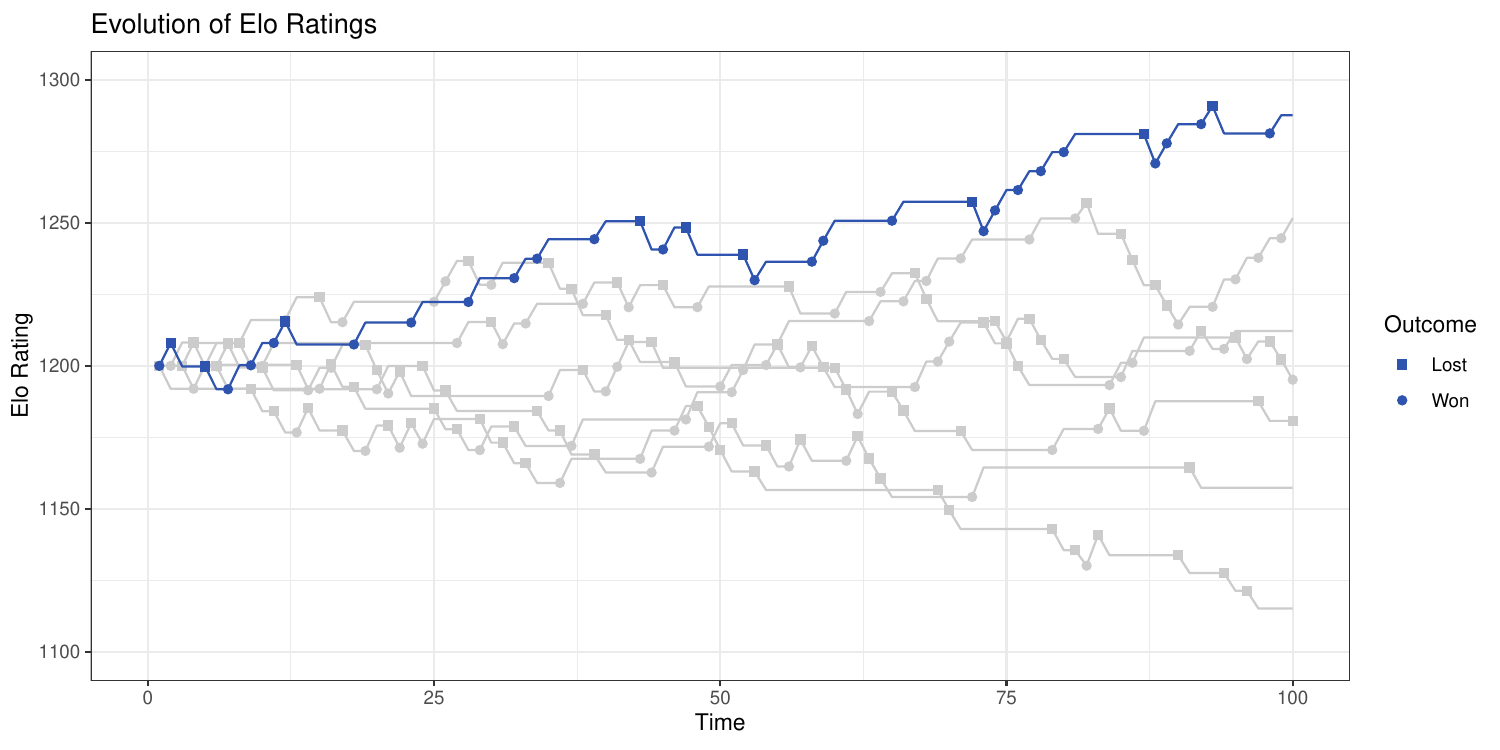}
\caption{Simulated paths of Elo ratings for seven players, with one player highlighted.}
\label{fig:example}
\end{figure}

\section{The Elo Rating System}
\label{sec:elo}

We present the standard Elo rating system first. Suppose there are $n$ players who play games against each other over time. Specifically, at time $t$, players $A$ and $B$ play a game in which one player wins and the other loses. The mechanism by which players $A$ and $B$ are selected to play at time $t$ is not considered here and we treat the pair $A$, $B$ as given. Let the outcome of the game be denoted by $y_t$, which can be
\begin{equation}
\label{eq:eloOutcome}
\begin{aligned}
\text{either } y_t &= \begin{pmatrix} y^{(A)}_t = 1 \\ y^{(B)}_t = 0 \\ \end{pmatrix} \text{ when player $A$ wins and player $B$ loses}, \\
\text{or } y_t &= \begin{pmatrix} y^{(A)}_t = 0 \\ y^{(B)}_t = 1 \\ \end{pmatrix} \text{ when player $A$ loses and player $B$ wins}. \\
\end{aligned}
\end{equation}

The goal is to assign a rating $r_t^{(i)}$ to each player $i \in \{ 1, 2, \ldots, n \}$ at each time $t \in \{ 1, 2, \ldots \}$. When the $t^{\text{th}}$ game takes place between players $A$ and $B$, and its outcome $y_t$ is observed, the ratings are updated using the formulas
\begin{equation}
\label{eq:eloUpdate}
\begin{aligned}
r_{t+1}^{(A)} &= r_{t}^{(A)} + 16 \left( y_t^{(A)} - \frac{1}{1 + 10^{- \left( r_{t}^{(A)} - r_{t}^{(B)} \right) / 400}} \right), \\
r_{t+1}^{(B)} &= r_{t}^{(B)} + 16 \left( y_t^{(B)} - \frac{1}{1 + 10^{\left( r_{t}^{(A)} - r_{t}^{(B)} \right) / 400}} \right), \\
r_{t+1}^{(j)} &= r_{t}^{(j)}, & \text{for } j \not\in \{A, B \}. \\
\end{aligned}
\end{equation}
The constant of 16 is called the K-factor, and it controls how quickly a player's rating changes after each game. Its value is chosen by convention; for example, 16 is often used for advanced players, while 32 may be used for weaker or provisional players. The scaling constant of 400 and the logarithm base of 10 are also chosen by convention. Initially, at time $t = 1$, all players are assigned the same rating
\begin{equation}
\label{eq:eloInit}
r_1^{(1)} = \cdots = r_1^{(n)} = 1~200.
\end{equation}
Again, this is an arbitrarily chosen constant. An example of the evolution of Elo ratings over time is illustrated in Figure \ref{fig:example}.

The rationale behind the Elo rating system lies in the assumption that the outcome of a game between players $A$ and $B$ is a random variable $Y_t$ with probability mass function
\begin{equation}
\label{eq:eloProb}
\begin{aligned}
\mathrm{P} \left[ Y_t = (1, 0)^\intercal \mid r_{t}^{(A)}, r_{t}^{(B)} \right] &= \mathrm{E} \left[ Y_t^{(A)} \mid r_{t}^{(A)}, r_{t}^{(B)} \right] = \frac{1}{1 + 10^{-\left( r_{t}^{(A)} - r_{t}^{(B)} \right) / 400}}, \\
\mathrm{P} \left[ Y_t = (0, 1)^\intercal \mid r_{t}^{(A)}, r_{t}^{(B)} \right] &= \mathrm{E} \left[ Y_t^{(B)} \mid r_{t}^{(A)}, r_{t}^{(B)} \right] = \frac{1}{1 + 10^{\left( r_{t}^{(A)} - r_{t}^{(B)} \right) / 400}}. \\
\end{aligned}
\end{equation}
As a result, \eqref{eq:eloUpdate} can be rewritten as
\begin{equation}
\label{eq:eloUpdateExpect}
\begin{aligned}
r_{t+1}^{(A)} &= r_{t}^{(A)} + 16 \left( y_t^{(A)} - \mathrm{E} \left[ Y_t^{(A)} \mid r_{t}^{(A)}, r_{t}^{(B)} \right] \right), \\
r_{t+1}^{(B)} &= r_{t}^{(B)} + 16 \left( y_t^{(B)} - \mathrm{E} \left[ Y_t^{(B)} \mid r_{t}^{(A)}, r_{t}^{(B)} \right] \right), \\
r_{t+1}^{(j)} &= r_{t}^{(j)}, & \text{for } j \not\in \{A, B \}.  \\
\end{aligned}
\end{equation}

\begin{remark}
The assumption \eqref{eq:eloProb} is a crucial assumption underlying the Elo rating system. Of course, in reality, it is not clear whether this assumption holds (although it is testable using the observed outcomes $y_t$). In other words: \emph{Is it really true that if the Elo rating system assigns players the ratings $r_{t}^{(A)}$ and $r_{t}^{(B)}$, then the true probabilities are indeed given by \eqref{eq:eloProb}? Are the true probabilities actually a function of the Elo ratings $r_{t}^{(A)}$ and $r_{t}^{(B)}$? Should a chess player, before a match, look at the Elo rating and think, ``Well, Elo says my rating is low, so I am going to play worse?'' What if the Elo ratings are incorrect?} In Section \ref{sec:dyn}, we show that the problem with this strong assumption need not be that serious.
\end{remark}

\begin{remark}
In Elo terminology, $y_t^{(i)}$ is called the actual score of player $i$, while $\mathrm{E} \left[ Y_t^{(i)} \right]$ is referred to as the expected score (see \citealp{Elo1978}). In statistics, however, the term score conventionally denotes the gradient of the log-likelihood. Throughout this paper, we adopt the statistical convention and use the term score exclusively to refer to the gradient of the log-likelihood. We refer to $y_t^{(i)}$ and $\mathrm{E} \left[ Y_t^{(i)} \right]$ as the observed and the expected outcomes, respectively.
\end{remark}

\section{The Score-Driven Rating System}
\label{sec:score}

Let us now establish a general score-driven rating system. We allow for any game outcome that can be expressed as a univariate or multivariate, discrete or continuous, random variable $Y_t \in \mathcal{Y}$. Let $y_t$ denote its observed value. Our goal remains to assign a rating $r_t = \left( r_t^{(1)}, \ldots, r_t^{(n)} \right)^\intercal$ to all players at each time $t$. We assume that the distribution of $Y_t$ has a known structure and depends on the parameters $r_t$. Let $f(y_t \mid r_t)$ denote the probability mass function in the case of discrete $Y_t$, or the density function in the case of continuous $Y_t$. Furthermore, let $f(y_t \mid r_t)$ be differentiable with respect to $r_t$, and have support independent of $r_t$.

The score-driving rating is defined as
\begin{equation}
\label{eq:sdUpdate}
r_{t+1}^{(i)} = r_{t}^{(i)} + K \nabla_i(r_t; y_t), \quad \text{for } i = 1, \ldots, n,
\end{equation}
where $K > 0$ is the K-factor and $\nabla_i(r_t; y_t)$ is the score for the parameter $r_{t}^{(i)}$, given by
\begin{equation}
\label{eq:sdScore}
\nabla_i(r_t; y_t) = \frac{\partial \ln f \left( y_t \mid r_t \right)}{\partial r_{t}^{(i)}}.
\end{equation}
The initial ratings are set to
\begin{equation}
\label{eq:sdInit}
r_1^{(1)} = \cdots = r_1^{(n)} = r_{\text{init}}.
\end{equation}
The difference from the Elo rating system presented in Section \ref{sec:elo} is that the score-driven rating system uses $\nabla_i(r_t; y_t)$ as the update term, whereas \eqref{eq:eloUpdateExpect} uses the difference between the observed and expected outcomes, $y_t^{(i)} - \mathrm{E} \left[ Y_t^{(i)} \right]$.

The score quantifies how the ratings should be adjusted to account for the observed game outcome $y_t$. It is the optimal updating term for the ratings in terms of maximizing the log-likelihood, as it points in the direction of the steepest ascent, providing the most natural local adjustment of the ratings given $y_t$. The score takes into account the shape of the entire distribution, not just the first moment. The first moment might not even exist or be meaningful, e.g., for categorical game outcomes (see \citealp{Szczecinski2022}). If $\nabla_i(r_t; y_t)$ is positive, then player $i$ overperformed in the $t^{\text{th}}$ game, and their future rating $r_{t+1}^{(i)}$ should be increased accordingly. In contrast, a negative $\nabla_i(r_t; y_t)$ indicates that player $i$ underperformed, resulting in a decrease of $r_{t+1}^{(i)}$. A zero $\nabla_i(r_t; y_t)$ indicates that player $i$ performed as expected, or that player $i$ did not participate in the $t^{\text{th}}$ game at all, in which case $\nabla_i(r_t; y_t)$ is also zero.

\begin{mainass}
For the score to serve as a meaningful update term with desirable properties, we require $f(y_t \mid r_t)$ to be a function of differences between ratings only, to be strictly log-concave, and to satisfy certain regularity conditions. Under these assumptions, the rating update has zero expected value, the sum of updates across all players is zero, and the update decreases as the player's rating increases. We examine these properties in detail in Section \ref{sec:prop}. But first, we compare the score-driven rating for the win/loss outcome with the traditional Elo rating.
\end{mainass}

\begin{example}[Win/Loss Between Two Players]
\label{exm:bernoulli}
For the simple win/loss outcome between two players, as discussed in Section \ref{sec:elo}, we can have
\begin{equation}
\label{eq:bernoulliProb}
f \left( y_t \mid r_{t}^{(A)}, r_{t}^{(B)} \right) = \frac{1}{1 + \exp \left( \alpha (-1)^{y_t^{(A)}} \left(r_{t}^{(A)} - r_{t}^{(B)} \right) \right)},
\end{equation}
where $\alpha > 0$ is the scaling constant. It is clear that $f$ depends on the ratings only through their difference, $r_{t}^{(A)} - r_{t}^{(B)}$, and it can be shown that it is strictly log-concave with respect to $r_{t}^{(A)}$ and $r_{t}^{(B)}$. The score is given by
\begin{equation}
\label{eq:bernoulliScore}
\begin{aligned}
\nabla_A \left( r_{t}^{(A)}, r_{t}^{(B)}; y_t \right) &= \frac{-\alpha (-1)^{y_t^{(A)}}}{1 + \exp \left( - \alpha (-1)^{y_t^{(A)}} \left(r_{t}^{(A)} - r_{t}^{(B)} \right) \right)}, \\
\nabla_B \left( r_{t}^{(A)}, r_{t}^{(B)}; y_t \right) &= \frac{\alpha (-1)^{y_t^{(A)}}}{1 + \exp \left( - \alpha (-1)^{y_t^{(A)}} \left(r_{t}^{(A)} - r_{t}^{(B)} \right) \right)}, \\
\nabla_j \left( r_{t}^{(A)}, r_{t}^{(B)}; y_t \right) &= 0, & \text{for } j \not\in \{A, B \}. \\
\end{aligned}
\end{equation}
The score function is illustrated in Figure \ref{fig:bernoulliScore}. The score is bounded above by $\alpha$ and below by $-\alpha$. It is positive if the player wins and negative if the player loses.
\end{example}

\begin{figure}
\centering
\includegraphics[scale=0.55]{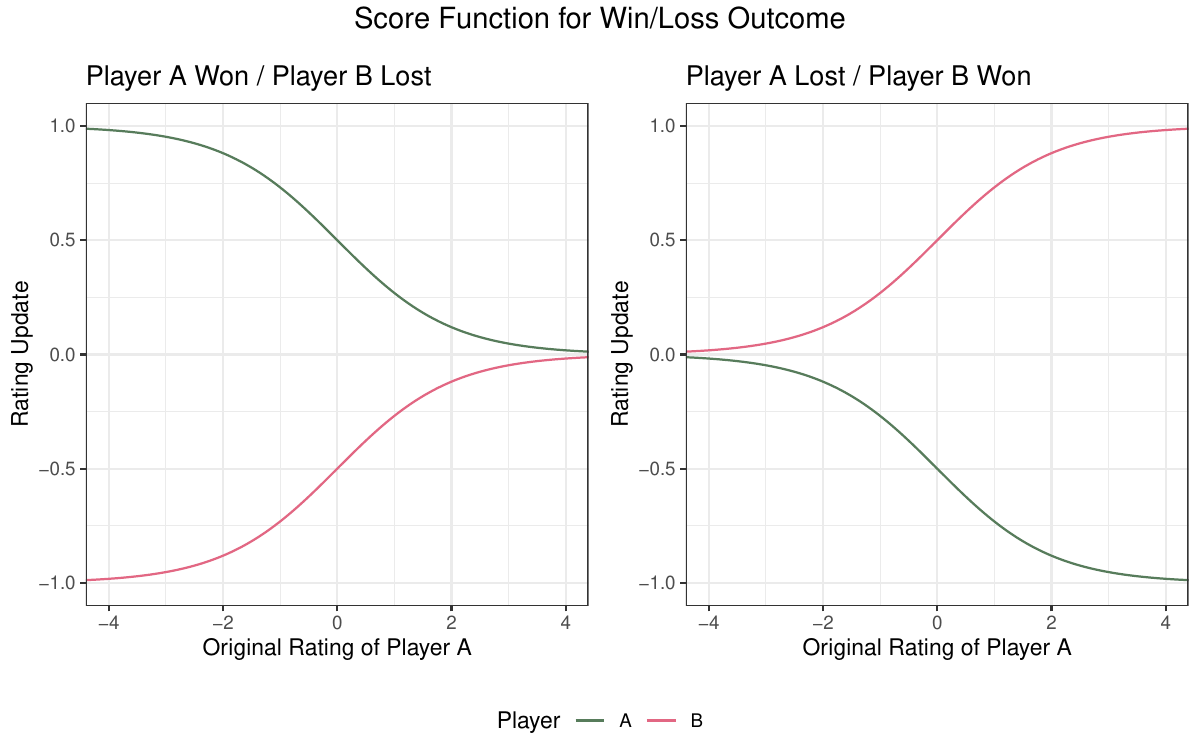}
\caption{The score function for a win/loss game outcome, modeled by \eqref{eq:bernoulliProb} with $\alpha = 1$. The rating for player $A$ is displayed on the x-axis, while it is fixed at 0 for player $B$.}
\label{fig:bernoulliScore}
\end{figure}

\section{Revisiting the Elo Rating}
\label{sec:equi}

Formulas \eqref{eq:eloUpdate} illustrate how the Elo rating system is usually described in the sports literature. For our convenience, it will be useful to write the same idea using a slightly more general notation. Let $\Lambda(\xi)$ be a link function that is smooth and increasing, with $\Lambda(\xi) \to 0$ as $\xi \to -\infty$ and $\Lambda(\xi) \to 1$ as $\xi \to \infty$, and that satisfies the symmetry property
\begin{equation}
\label{eq:symmetry}
\Lambda(-\xi) + \Lambda(\xi) = 1.
\end{equation}
The $\Lambda$-Elo rating system assumes 
\begin{equation}
\label{eq:eloProbLambda}
\begin{aligned}
\mathrm{P} \left[ Y_t = (1, 0)^\intercal \mid r_{t}^{(A)}, r_{t}^{(B)} \right] &= \Lambda \left( \tilde{\alpha} \left( r_{t}^{(A)} - r_{t}^{(B)} \right) \right), \\
\mathrm{P} \left[ Y_t = (0, 1)^\intercal \mid r_{t}^{(A)}, r_{t}^{(B)} \right] &= \Lambda \left( - \tilde{\alpha} \left( r_{t}^{(A)} - r_{t}^{(B)} \right) \right). \\
\end{aligned}
\end{equation}
The $\Lambda$-Elo ratings then follow
\begin{equation}
\label{eq:eloUpdateLambda}
\begin{aligned}
r_{t+1}^{(A)} &= r_{t}^{(A)} + \tilde{K} \left( y_t^{(A)} - \Lambda \left( \tilde{\alpha} \left( r_{t}^{(A)} - r_{t}^{(B)} \right) \right) \right), \\
r_{t+1}^{(B)} &= r_{t}^{(B)} + \tilde{K} \left( y_t^{(B)} - \Lambda \left( - \tilde{\alpha} \left( r_{t}^{(A)} - r_{t}^{(B)} \right) \right) \right), \\
r_{t+1}^{(j)} &= r_{t}^{(j)}, & \text{for } j \not\in \{A, B \}.  \\
\end{aligned}
\end{equation}
It is easy to see that \eqref{eq:eloUpdate} falls into this setup, for example when
\begin{equation}
\tilde{K} = 16, \qquad \tilde{\alpha} = \frac{\ln 10}{400}, \qquad \Lambda(\xi) = \frac{1}{1 + e^{-\xi}}.
\end{equation}
The logistic link function clearly satisfies the symmetry property \eqref{eq:symmetry}.

Originally, however, the Elo rating system was based on the normal distribution (see \citealp{Elo1978}). Such system also falls within the $\Lambda$-Elo framework, with $\Lambda(\xi)$ representing the cumulative distribution function of the normal distribution. The logistic function was later adopted purely for convenience, as an easy-to-compute approximation of the normal cumulative distribution function. In the following property, we show that the now-widely used logistic link has an additional motivation: the Elo rating system is a special case of the score-driven rating system if and only if the logistic link is employed.

\begin{proposition}[Elo Rating with Logistic Link Is Score-Driven Rating]
Assuming the probability mass function takes the form \eqref{eq:eloProbLambda}, the $\Lambda$-Elo rating system \eqref{eq:eloUpdateLambda} coincides with the score-driven rating system \eqref{eq:sdUpdate} if and only if $\Lambda(\xi)$ is the logistic function.
\end{proposition}

\begin{proof}
First, we show that \eqref{eq:sdUpdate} corresponds to \eqref{eq:eloUpdateLambda} when the logistic link function is used. The logistic function satisfies
\begin{equation}
\label{eq:logitDeriv1}
\frac{\partial \Lambda(\xi)}{\partial \xi} = \Lambda(\xi) \left(1 - \Lambda(\xi) \right) \qquad \text{and} \qquad \frac{\partial \ln \Lambda(\xi)}{\partial \xi} = 1 - \Lambda(\xi).
\end{equation}
Using the symmetry property \eqref{eq:symmetry}, we can further rewrite this as
\begin{equation}
\label{eq:logitDeriv2}
\frac{\partial \Lambda(\xi)}{\partial \xi} = \Lambda(\xi) \Lambda(-\xi) \qquad \text{and} \qquad \frac{\partial \ln \Lambda(\xi)}{\partial \xi} = \Lambda(-\xi).
\end{equation}
If player $A$ wins, then
\begin{equation}
\label{eq:equalElo1}
\begin{aligned}
\nabla_A \left( r_{t}^{(A)}, r_{t}^{(B)}; y_t \right) &= \frac{\partial \ln \Lambda \left( \alpha \left( r_{t}^{(A)} - r_{t}^{(B)} \right) \right)}{\partial r_{t}^{(A)}} = \alpha \left( 1 - \Lambda \left( \alpha \left( r_{t}^{(A)} - r_{t}^{(B)} \right) \right) \right), \\
\nabla_B \left( r_{t}^{(A)}, r_{t}^{(B)}; y_t \right)  &= \frac{\partial \ln \Lambda \left( \alpha \left( r_{t}^{(A)} - r_{t}^{(B)} \right) \right)}{\partial r_{t}^{(B)}} = \alpha \left( - \Lambda \left( - \alpha \left( r_{t}^{(A)} - r_{t}^{(B)} \right) \right) \right), \\
\end{aligned}
\end{equation}
where we have used \eqref{eq:logitDeriv1} and \eqref{eq:logitDeriv2}, respectively. On the other hand, if player $A$ loses, then
\begin{equation}
\label{eq:equalElo2}
\begin{aligned}
\nabla_A \left( r_{t}^{(A)}, r_{t}^{(B)}; y_t \right) &= \frac{\partial \ln \Lambda \left( - \alpha \left( r_{t}^{(A)} - r_{t}^{(B)} \right) \right)}{\partial r_{t}^{(A)}} = \alpha \left( - \Lambda \left( \alpha \left( r_{t}^{(A)} - r_{t}^{(B)} \right) \right) \right), \\
\nabla_B \left( r_{t}^{(A)}, r_{t}^{(B)}; y_t \right)  &= \frac{\partial \ln \Lambda \left( - \alpha \left( r_{t}^{(A)} - r_{t}^{(B)} \right) \right)}{\partial r_{t}^{(B)}} = \alpha \left( 1 - \Lambda \left( - \alpha \left( r_{t}^{(A)} - r_{t}^{(B)} \right) \right) \right). \\
\end{aligned}
\end{equation}
By setting $\tilde{K} = \alpha K$ and $\tilde{\alpha} = \alpha$, we obtain equality between \eqref{eq:sdUpdate} and \eqref{eq:eloUpdateLambda}.

The crucial property of the link function $\Lambda(\xi)$ required for \eqref{eq:equalElo1} and \eqref{eq:equalElo2} to hold is given in \eqref{eq:logitDeriv1}. However, \eqref{eq:logitDeriv1} is the well-known logistic differential equation, whose only solution is the logistic function. Thus, no other choice of $\Lambda(\xi)$ can produce equality between \eqref{eq:sdUpdate} and \eqref{eq:eloUpdateLambda}.
\end{proof}

\section{Other Notable Examples}
\label{sec:exm}

We present three directions in which the classical Elo rating system can be generalized or adjusted using the score-driven approach. Although variants of these models exist in the literature, we derive and present them here within a unified framework.

\begin{example}[Margin of Victory Between Two Players]
\label{exm:skellam}
In sports such as football, basketball, or ice hockey, games may end in a win, loss, or draw, with outcomes determined by the points scored by each team. One way to generalize the Elo rating system is to account for the margin of victory---\emph{was the win decisive or narrow?}

We let the results of games be represented by the scored points. Following the notation from previous sections, the observed outcome of game $t$ between players $A$ and $B$ can be denoted as
\begin{equation}
\label{eq:skellamOutcome}
y_t = \left( y_t^{(A)}, y_t^{(B)} \right)^\intercal, \qquad y_t^{(A)}, y_t^{(B)} \in \mathbb{Z},
\end{equation}
where $y_t^{(i)}$ denotes the scored points of player $i$. This form of outcome is richer than we need, as this example focuses only on the difference between the points.

We assume that $Y_t^{(A)} - Y_t^{(B)}$ follows the Skellam distribution, which describes the difference between two independent Poisson-distributed variables (see \citealp{Skellam1946}). In our case, the two variables represent the number of points scored by players $A$ and $B$, respectively, with rates $\lambda_A$ and $\lambda_B$ given by
\begin{equation}
\label{eq:skellamRate}
\lambda_A = \exp \left(\alpha \left( r_{t}^{(A)} - r_{t}^{(B)} \right) \right) \qquad \text{and} \qquad \lambda_B = \exp \left(-\alpha \left( r_{t}^{(A)} - r_{t}^{(B)} \right) \right).
\end{equation}
The probability mass function is then
\begin{equation}
\label{eq:skellamProb}
f \left( y_t \mid r_{t}^{(A)}, r_{t}^{(B)} \right) = \exp \left( \alpha \left( r_{t}^{(A)} - r_{t}^{(B)} \right) \left( y_t^{(A)} - y_t^{(B)} \right) - 2 \cosh \left( \alpha \left( r_{t}^{(A)} - r_{t}^{(B)} \right) \right) \right) I_{y_t^{(A)} - y_t^{(B)}}(2),
\end{equation}
where $I_{\cdot}(\cdot)$ is the modified Bessel function of the first kind. This probability mass function satisfies our assumptions, as it is log-concave and depends on the ratings only through their difference. The first two moments are
\begin{equation}
\label{eq:skellamMoment}
\begin{aligned}
\mathrm{E} \left[ Y_t^{(A)} - Y_t^{(B)} \mid r_{t}^{(A)}, r_{t}^{(B)} \right] &= 2 \sinh \left( \alpha \left( r_{t}^{(A)} - r_{t}^{(B)} \right) \right), \\
\mathrm{var} \left[ Y_t^{(A)} - Y_t^{(B)} \mid r_{t}^{(A)}, r_{t}^{(B)} \right] &= 2 \cosh \left( \alpha \left( r_{t}^{(A)} - r_{t}^{(B)} \right) \right).
\end{aligned}
\end{equation}
The score is given by
\begin{equation}
\label{eq:skellamScore}
\begin{aligned}
\nabla_A \left( r_{t}^{(A)}, r_{t}^{(B)}; y_t \right) &= \alpha \left( y_t^{(A)} - y_t^{(B)} - 2 \sinh \left( \alpha \left( r_{t}^{(A)} - r_{t}^{(B)} \right) \right) \right), \\
\nabla_B \left( r_{t}^{(A)}, r_{t}^{(B)}; y_t \right) &= - \nabla_A \left( r_{t}^{(A)}, r_{t}^{(B)}; y_t \right), \\
\nabla_j \left( r_{t}^{(A)}, r_{t}^{(B)}; y_t \right) &= 0, & \text{for } j \not\in \{A, B \}. \\
\end{aligned}
\end{equation}
The score is unbounded, as illustrated in Figure \ref{fig:skellamScore}. Interestingly, a player may lose rating even after winning if the margin of victory is small and the opponent's rating is substantially lower, since the model expects a more decisive win.

A score-driven model based on the Skellam distribution was employed by \cite{Koopman2019} and \cite{Lasek2021}. Both papers also consider a more general model based on the bivariate Poisson distribution, which introduces positive correlation between $Y_t^{(A)}$ and $Y_t^{(B)}$ through an additional parameter. This model further allows for a distinction between the offensive and defensive ratings of each player. However, when the model is formulated using a single rating per player, as in \eqref{eq:skellamRate}, the score coincides with \eqref{eq:skellamScore}.
\end{example}


\begin{figure}
\centering
\includegraphics[scale=0.55]{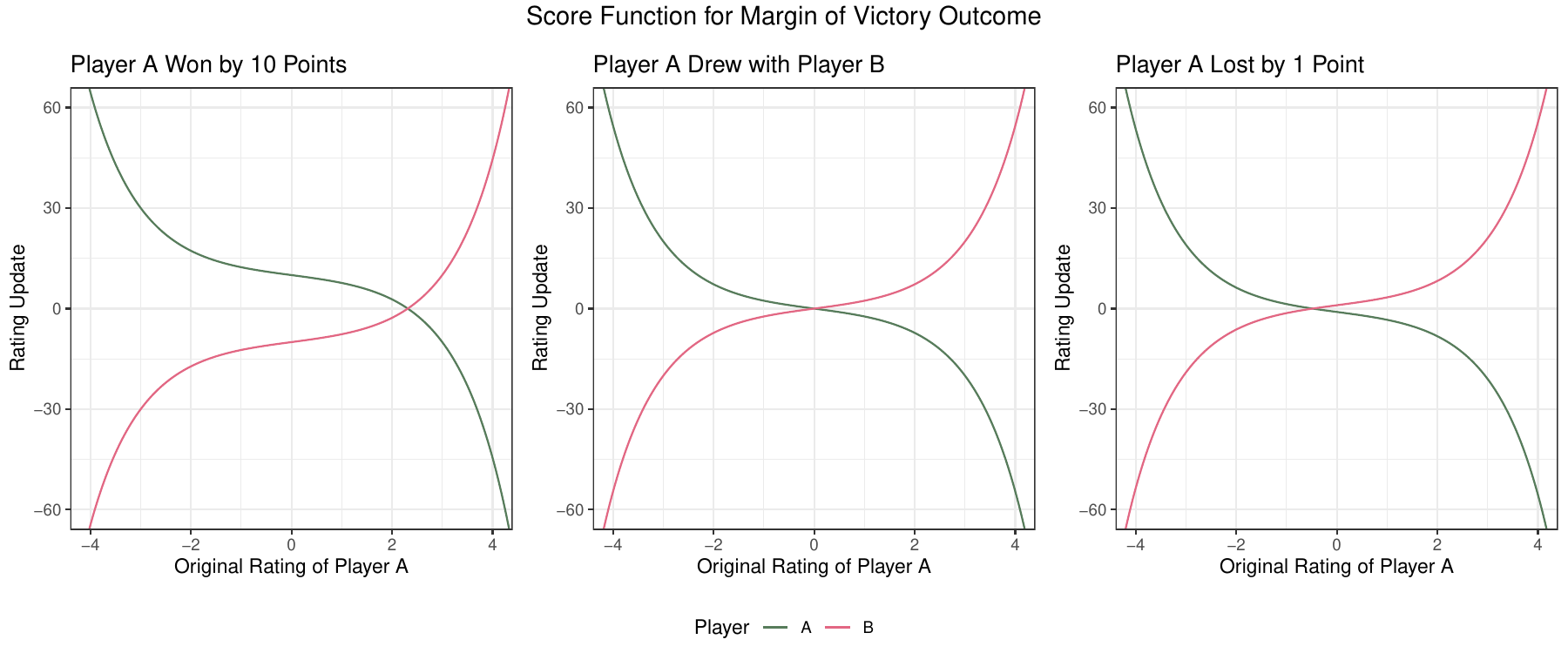}
\caption{The score function for a margin of victory game outcome, modeled by \eqref{eq:skellamProb} with $\alpha = 1$. The rating for player $A$ is displayed on the x-axis, while it is fixed at 0 for player $B$.}
\label{fig:skellamScore}
\end{figure}

\begin{example}[Win/Draw/Loss Between Two Players]
\label{exm:ordit}

In many tournament formats---such as those used in football, basketball, or ice hockey---only the outcome in the form of a win, draw, or loss can be of primary relevance, regardless of the point difference. Moreover, other sports and games that do not rely on points can also result in a draw. For instance, chess provides several distinct mechanisms through which a game may end in a draw.

We extend the win/loss outcome \eqref{eq:eloOutcome} from the standard Elo rating system to be
\begin{equation}
\label{eq:orditOutcome}
\begin{aligned}
\text{either } y_t &= \begin{pmatrix} y^{(A)}_t = 1 \\ y^{(B)}_t = 0 \\ \end{pmatrix} \text{ when player $A$ wins and player $B$ loses}, \\
\text{or } y_t &= \begin{pmatrix} y^{(A)}_t = 0.5 \\ y^{(B)}_t = 0.5 \\ \end{pmatrix} \text{ when there is a draw between players $A$ and $B$}, \\
\text{or } y_t &= \begin{pmatrix} y^{(A)}_t = 0 \\ y^{(B)}_t = 1 \\ \end{pmatrix} \text{ when player $A$ loses and player $B$ wins}. \\
\end{aligned}
\end{equation}
The numerical values of  $y_t^{(i)}$ themselves are not important in our case; only the order of the three categories matters. Using the score as the update term enables this, in contrast to relying on the expected outcome.

We incorporate draws using the ordered logit model, also known as the proportional odds logistic regression model (see \citealp{McCullagh1980}). For this, an additional parameter $\delta \geq 0$, determining the draw threshold, is required. The probability mass function is then
\begin{equation}
\label{eq:orditProb}
\begin{aligned}
f \left( y_t \mid r_{t}^{(A)}, r_{t}^{(B)} \right) &= \begin{cases} 
1 - \frac{1}{1 + \exp \left( - \delta + \alpha \left( r_t^{(A)} - r_t^{(B)} \right) \right)}, & \text{for } y_t = (1, 0)^\intercal, \\
\frac{1}{1 + \exp \left( - \delta + \alpha \left( r_t^{(A)} - r_t^{(B)} \right) \right)} - \frac{1}{1 + \exp \left( \delta + \alpha \left( r_t^{(A)} - r_t^{(B)} \right) \right)}, & \text{for } y_t = (0.5, 0.5)^\intercal, \\
\frac{1}{1 + \exp \left( \delta + \alpha \left( r_t^{(A)} - r_t^{(B)} \right) \right)}, & \text{for } y_t = (0, 1)^\intercal. \\
\end{cases}
\end{aligned}
\end{equation}
The larger the value of $\delta$, the greater the probability of a draw, with $\delta = 0$ corresponding to a zero probability of a draw. The probability mass function satisfies our required assumptions, being log-concave and dependent on the ratings only through their difference. The score is given by
\begin{equation}
\label{eq:orditScore}
\begin{aligned}
\nabla_A \left( r_{t}^{(A)}, r_{t}^{(B)}; y_t \right) &= \begin{cases} 
\alpha \frac{1}{1 + \exp \left( -\delta + \alpha \left( r_t^{(A)} - r_t^{(B)} \right) \right)}, & \text{for } y_t = (1, 0)^\intercal, \\
\alpha \frac{-\sinh \left( \alpha \left( r_t^{(A)} - r_t^{(B)} \right) \right)}{\cosh \left( \delta \right) + \cosh \left( \alpha \left( r_t^{(A)} - r_t^{(B)} \right) \right)}, & \text{for } y_t = (0.5, 0.5)^\intercal, \\
\alpha \frac{-1}{1 + \exp \left( -\delta - \alpha \left( r_t^{(A)} - r_t^{(B)} \right) \right)}, & \text{for } y_t = (0, 1)^\intercal, \\
\end{cases} \\
\nabla_B \left( r_{t}^{(A)}, r_{t}^{(B)}; y_t \right) &= - \nabla_A \left( r_{t}^{(A)}, r_{t}^{(B)}; y_t \right), \\
\nabla_j \left( r_{t}^{(A)}, r_{t}^{(B)}; y_t \right) &= 0, & \text{for } j \not\in \{A, B \}. \\
\end{aligned}
\end{equation}
As in Example \ref{exm:bernoulli}, the score lies in the interval $(0, \alpha)$ in the case of a win and in $(-\alpha, 0)$ in the case of a loss. In the event of a draw, it lies in $(-\alpha, \alpha)$, with the sign depending on the sign of the difference between the ratings, as shown in Figure \ref{fig:orditScore}.

This ordered logit model was used by \cite{Lasek2021}, whereas \cite{Koopman2019} employed an ordered probit model. However, draws can be incorporated into the underlying stochastic model in several different ways. For example, another model without an additional parameter can be derived from the assumption of the Skellam distribution for the point differences \eqref{eq:skellamProb}. Further models proposed in the literature include those by \cite{Szczecinski2020} and \cite{Glickman2025}.
\end{example}

\begin{figure}
\centering
\includegraphics[scale=0.55]{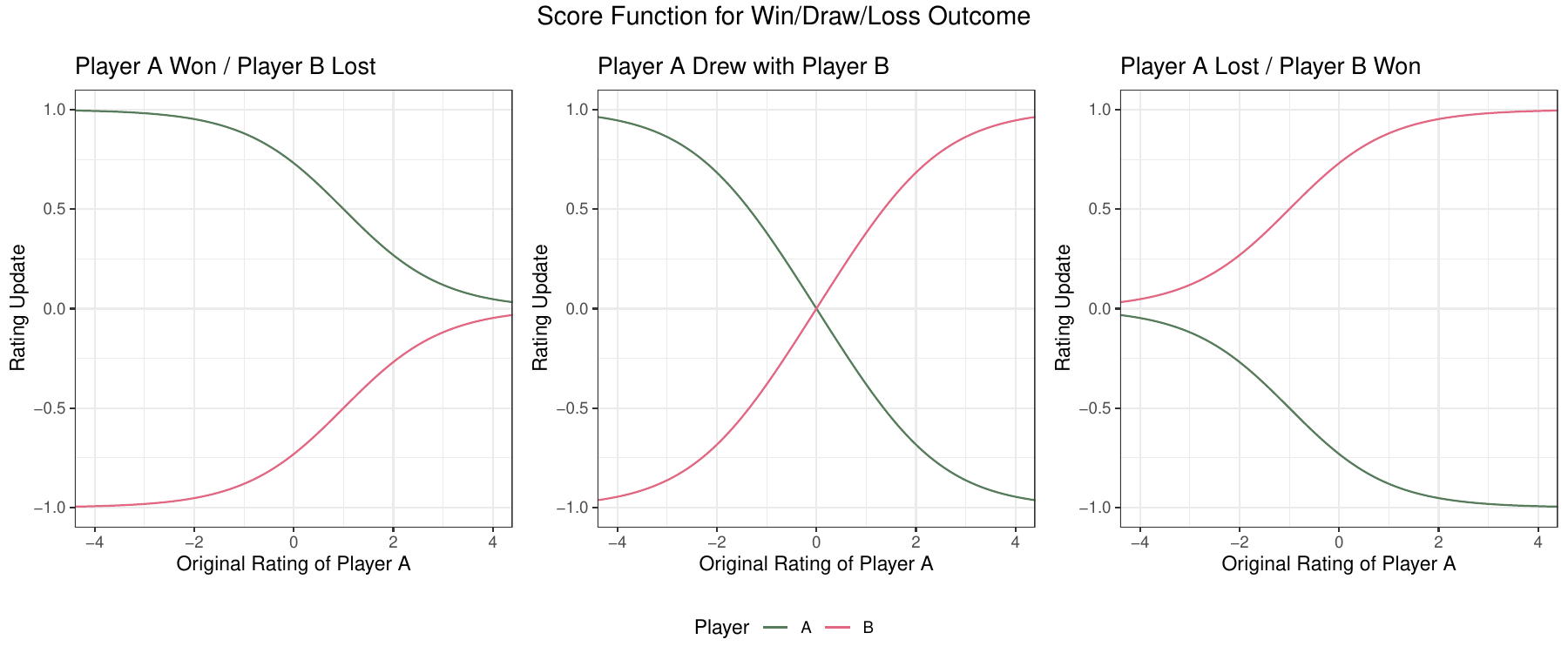}
\caption{The score function for a win/draw/loss game outcome, modeled by \eqref{eq:orditProb} with $\alpha = 1$ and $\delta = 1$. The rating for player $A$ is displayed on the x-axis, while it is fixed at 0 for player $B$.}
\label{fig:orditScore}
\end{figure}

\begin{example}[Ranking of Multiple Players]
\label{exm:pluce}
So far, we have considered outcomes of games between two players, $A$ and $B$. Many sports, such as running, swimming, cycling, or Formula racing, measure performance by time, from which the ranking of all participants is derived. Rankings are also produced in sports based on points, such as gymnastics or golf. Moreover, tournaments or leagues yield final standings---that is, the ranking of participants. In this example, we extend the win/loss outcome between two players to the ranking of multiple players.

Let us begin by defining the ranking outcome of games. We allow for the possibility that not all players participate. Let $\mathcal{M}_t$ denote the set of all participating players at time $t$, with $m_t = |\mathcal{M}_t|$ their number. The outcome is then the complete ranking of the participating teams, denoted as 
\begin{equation}
\label{eq:pluceOutcome}
y_t = \left( y_t^{(i)} : i \in \mathcal{M}_t \right), \qquad y_t^{(i)} \in \{ 1, \ldots, m_t \} \text{ for all } i, \qquad \text{and} \qquad y_t^{(j)} \neq y_t^{(k)} \text{ for all } j \neq k.
\end{equation}
Furthermore, we use the notation `$p$-th' for the player who has rank $p$, i.e., the number $i$ such that $y_t^{(i)} = p$. Note that this notation differs slightly from Example \ref{exm:bernoulli} with two players, where 1 denotes the first place (a win) and 0 denotes the second place (a loss).

We assume that the random ranking follows the Plackett--Luce distribution of \cite{Plackett1975} and \cite{Luce1959}. Its probability mass function is given by
\begin{equation}
\label{eq:pluceProb}
f \left( y_t \mid r_{t}^{(k)}, k \in \mathcal{M}_t \right) = \prod_{p=1}^{m_t} \frac{\exp \left( \alpha r_{t}^{(p\text{-th})} \right)}{\sum_{q = p}^{m_t} \exp \left( \alpha r_{t}^{(q\text{-th})} \right)}.
\end{equation}
In this model, the probability of placing first for each player $i$ is proportional to $\exp \left( \alpha r_{t}^{(i)} \right)$. The probability mass function is invariant under the addition of a constant to all ratings. We can thus subtract the rating of any player---for example, the one placing last---from all ratings and obtain
\begin{equation}
\label{eq:pluceProb2}
f \left( y_t \mid r_{t}^{(k)}, k \in \mathcal{M}_t \right) = \prod_{p=1}^{m_t-1} \frac{\exp \left( \alpha \left( r_{t}^{(p\text{-th})} - r_{t}^{(m_t\text{-th})} \right) \right)}{1 + \sum_{q = p}^{m_t-1} \exp \left( \alpha \left( r_{t}^{(q\text{-th})} - r_{t}^{(m_t\text{-th})} \right) \right)}.
\end{equation}
It is clear from \eqref{eq:pluceProb2} that the probability mass function depends on the ratings only through their difference. It can also be shown to be log-concave. When $m_t=2$, it reduces to \eqref{eq:bernoulliProb}, albeit with differently labeled outcomes, as noted above. The score is given by
\begin{equation}
\label{eq:pluceScore}
\begin{aligned}
\nabla_i \left( r_{t}^{(k)}, k \in \mathcal{M}_t; y_t \right) &= \alpha \left( 1 - \sum_{p=1}^{y_t^{(i)}} \frac{\exp \left( \alpha r_{t}^{(i)} \right)}{\sum_{q = p}^{m_t} \exp \left( \alpha r_{t}^{(q\text{-th})} \right)} \right), & & \text{for } i \in \mathcal{M}_t, \\
\nabla_j \left( r_{t}^{(k)}, k \in \mathcal{M}_t; y_t \right) &= 0, & & \text{for } j \not\in \mathcal{M}_t. \\
\end{aligned}
\end{equation}
The score lies in $(\alpha - \alpha r, \alpha)$ for players with ranks $r = 1, \ldots, m_t-1$, and in $(\alpha - \alpha m_t, 0)$ for the player ranked last (see \citealp{Holy2022f}). It is illustrated in Figure \ref{fig:pluceScore}.

A score-driven Plackett--Luce model was utilized by \cite{Holy2022f} and \cite{Holy2025a}. \cite{Holy2022f} also addressed the case of partial rankings, where only the top players are ranked, while the rest are unranked (but placed below them). Another extension would be to incorporate draws.

This example highlights an advantage of using the score. For some distributions, the score coincides with the expected value---for instance, in our Skellam model for the margin of victory. However, under the Plackett--Luce distribution, the expected value differs from the score and is not even available in closed form, but only as a sum over all permutations. The score is therefore more broadly applicable than the expected value.
\end{example}

\begin{figure}
\centering
\includegraphics[scale=0.55]{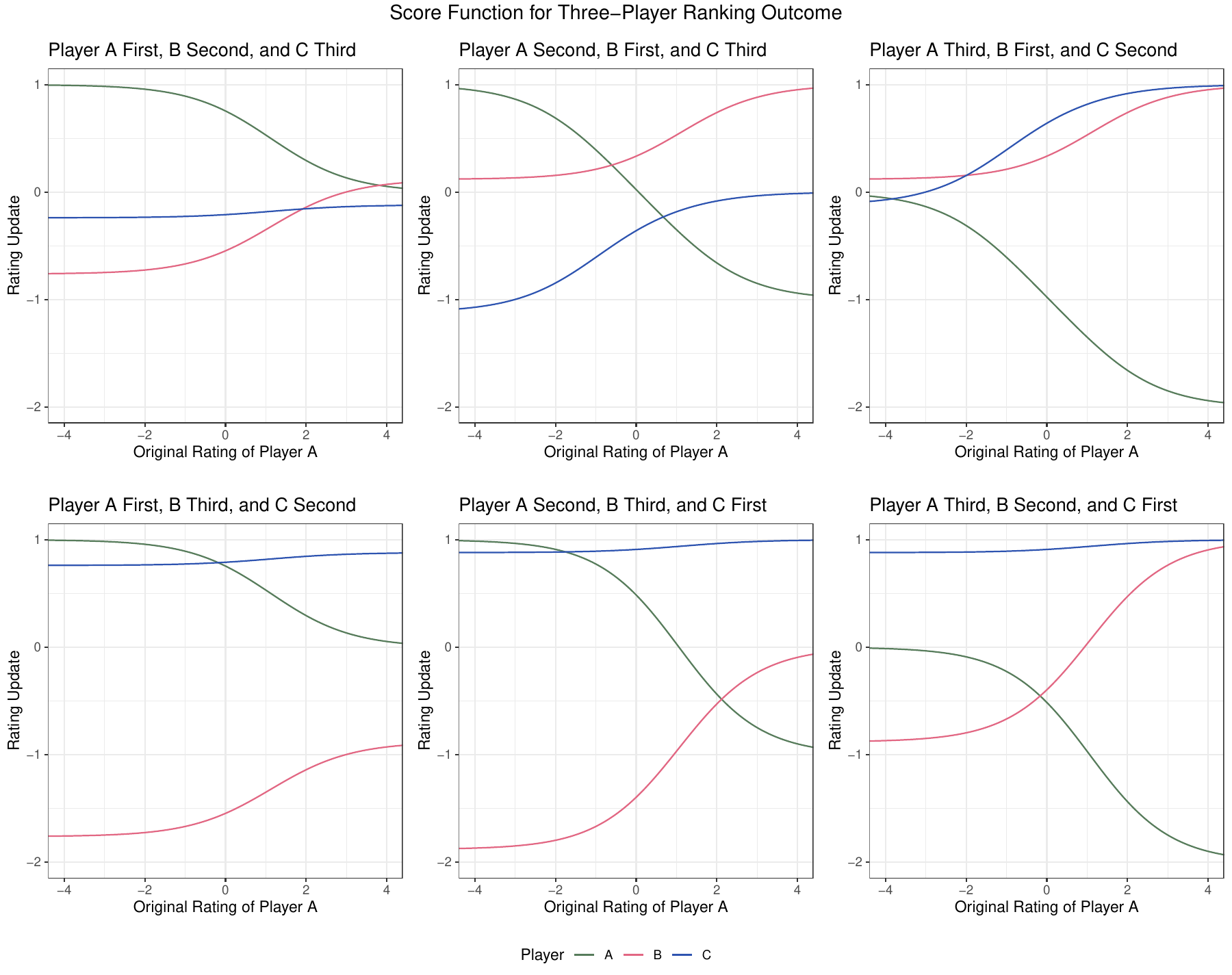}
\caption{The score function for a ranking game outcome with three players, modeled by \eqref{eq:pluceProb} with $\alpha = 1$. The rating for player $A$ is displayed on the x-axis, while it is fixed at $1$ for player $B$ and at $-1$ for player $C$.}
\label{fig:pluceScore}
\end{figure}

\section{Properties of the Score}
\label{sec:prop}

We show that the score has several important properties that make it suitable for updating ratings. First, the score has zero expected value. This means that, before a game, no player can expect an increase or decrease in their rating, regardless of their current rating or that of their opponent(s). Even a weak player facing a strong opponent has zero expected value, since the likely loss would reduce the rating only slightly, while a rare win would lead to a large increase. This property is based on the assumption that $f(y_t \mid r_t)$ is determined by $r_t$, i.e., that the ratings reflect the true skill levels of all players. We further discuss this strict assumption in Section \ref{sec:dyn}. The proof of this proposition is a classical result found in many statistics textbooks, see e.g.\ \cite{Schervish1995}.

\begin{proposition}[Zero Expected Score]
\label{prop:expected}
Let $f(y_t \mid r_t)$ be a differentiable function with support independent of $r_t$. Furthermore, assume that differentiation and integration of $f(y_t \mid r_t)$ are interchangeable. Then, the score has zero expected value, i.e.
\begin{equation}
\mathrm{E} \left[ \nabla_i(r_t; y_t) \mid r_t \right] = 0, \qquad i = 1, \ldots, n.
\end{equation}
\end{proposition}

\begin{proof}
For any $i$, it holds that
\begin{equation}
\begin{aligned}
\mathrm{E} \left[ \nabla_i(r_t; y_t) \mid r_t \right] &= \int_{\mathcal{Y}} \nabla_i(r_t; y_t) f(y_t \mid r_t) d y_t \\
&= \int_{\mathcal{Y}} \frac{\partial \ln f(y_t \mid r_t)}{\partial r_{t}^{(i)}} f(y_t \mid r_t) d y_t \\
&= \int_{\mathcal{Y}} \frac{\partial f(y_t \mid r_t)}{\partial r_{t}^{(i)}} \frac{1}{f(y_t \mid r_t)} f(y_t \mid r_t) d y_t \\
&= \int_{\mathcal{Y}} \frac{\partial f(y_t \mid r_t)}{\partial r_{t}^{(i)}} d y_t \\
&= \frac{\partial \int_{\mathcal{Y}} f(y_t \mid r_t) d y_t}{\partial r_{t}^{(i)}} \\
&= \frac{\partial 1}{\partial r_{t}^{(i)}} \\
&= 0.
\end{aligned}
\end{equation}
\end{proof}

Next, we show that the sum of scores over all players is zero for any game outcome. This means that, if the pool of players is stable over time, the sum of all ratings remains constant. There is no inflation or deflation in the ratings, and the average rating equals the value used for the initial ratings, $r_{\text{init}}$.

\begin{proposition}[Zero Sum of Scores]
\label{prop:sum}
Let $f(y_t \mid r_t)$ be a differentiable function with support independent of $r_t$. Furthermore, assume that $f(y_t \mid r_t)$ can be expressed as a
function of differences of ratings, i.e.
\begin{equation}
f(y_t \mid r_t) = f(y_t \mid d_t), \quad d_t = \left( r_t^{(1)} - r_t^{(n)}, \ldots, r_t^{(n-1)} - r_t^{(n)} \right)^\intercal.
\end{equation}
Then, the sum of all scores is zero for any outcome, i.e.
\begin{equation}
\sum_{i=1}^n \nabla_i(r_t; y_t) = 0, \qquad y_t \in \mathcal{Y}.
\end{equation}
\end{proposition}

\begin{proof}
Using the chain rule, it holds that
\begin{equation}
\nabla_i(r_t; y_t)
= \frac{\partial \ln f(y_t \mid r_t)}{\partial r_t^{(i)}}
= \left( \frac{\partial \ln f(y_t \mid r_t)}{\partial \left( r_t^{(1)} - r_t^{(n)} \right)}, \ldots, \frac{\partial \ln f(y_t \mid r_t)}{\partial \left( r_t^{(n-1)} - r_t^{(n)} \right)} \right)
\begin{pmatrix}
\frac{\partial \left( r_t^{(1)} - r_t^{(n)} \right)}{\partial r_t^{(i)}} \\
\vdots \\
\frac{\partial \left( r_t^{(n-1)} - r_t^{(n)} \right)}{\partial r_t^{(i)}}
\end{pmatrix}.
\end{equation}
For $i = 1, \ldots, n-1$, this equals simply
\begin{equation}
\nabla_i(r_t; y_t) = \frac{\partial \ln f(y_t \mid r_t)}{\partial \left( r_t^{(i)} - r_t^{(n)} \right)},
\end{equation}
while for $i=n$, it is
\begin{equation}
\nabla_n(r_t; y_t) = - \sum_{j=1}^{n-1} \frac{\partial \ln f(y_t \mid r_t)}{\partial \left( r_t^{(j)} - r_t^{(n)} \right)}.
\end{equation}
Therefore, we have
\begin{equation}
\sum_{i=1}^n \nabla_i(r_t; y_t) = \sum_{j=1}^{n-1} \frac{\partial \ln f(y_t \mid r_t)}{\partial \left( r_t^{(i)} - r_t^{(n)} \right)} - \sum_{j=1}^{n-1} \frac{\partial \ln f(y_t \mid r_t)}{\partial \left( r_t^{(j)} - r_t^{(n)} \right)} = 0.
\end{equation}
\end{proof}

Another important property is that the score for player $i$ is a decreasing function of their own rating. For example, a win against a mediocre opponent results in a large increase in rating for a weaker player, but only a modest increase for an already strong player. Conversely, a loss against a mediocre opponent leads to a small decrease in rating for a weaker player, but a considerable decrease for a strong player. This property holds under the assumption that $\ln f(y_t \mid r_t)$ is strictly concave. Many commonly used probability distributions, such as those in the exponential family, are log-concave. Log-concavity ensures unimodality and guarantees that any local maximum is also the global maximum in maximum likelihood estimation.

\begin{proposition}[Decreasing Score for Own Rating Increasing]
\label{prop:decrease}
Let $f(y_t \mid r_t)$ be a twice differentiable function with support independent of $r_t$. Furthermore, assume that $f(y_t \mid r_t)$ is strictly log-concave. Then, $\nabla_i(r_t; y_t)$ is decreasing as $r_t^{(i)}$ increases, $i = 1, \ldots, n$.
\end{proposition}

\begin{proof}
As $\ln f(y_t \mid r_t)$ is strictly concave, it holds that
\begin{equation}
\frac{\partial \nabla_n(r_t; y_t)}{\partial r_t^{(i)}} = \frac{\partial^2 \ln f(y_t \mid r_t)}{\partial \left( r_t^{(i)} \right)^2} < 0, \qquad i=1,\ldots,n.
\end{equation}
The score $\nabla_n(r_t; y_t)$ as a function of $r_t^{(i)}$ is therefore decreasing.
\end{proof}


\section{The Reversion Dynamics}
\label{sec:dyn}

Both the Elo and score-driven rating systems presented above were derived under the assumption that the specified probability model is the true underlying process generating the game outcomes, see, e.g., \eqref{eq:eloProb}. As noted in Section \ref{sec:elo}, it is more realistic to relax this assumption and suppose that players do not perform according to their assigned ratings, but rather according to their true skill levels, which may considerably differ from the perceived ratings. For example, at time $t=1$, all players are assigned the initial rating  $r_{\text{init}}$, which is not connected to their performance in any way. The main issue is that the true skills of the players are not directly observed, but can be inferred only through game outcomes.

In this section, we generalize our setting. We assume that each player $i$ has an unobserved, time-varying true skill $s_t^{(i)}$, $t = 1, 2, \ldots$. Furthermore, we assume that the structure of the underlying probability model of game outcomes is known, but it depends on the unknown skill levels $s_t = \left( s_t^{(1)}, \ldots, s_t^{(n)} \right)^\intercal$ instead of $r_t$. The corresponding probability mass function (or density function) is then $f(y_t \mid s_t)$. Following Properties \ref{prop:expected}--\ref{prop:decrease}, we assume that it is a twice differentiable, strictly log-concave function with support independent of $s_t$, and that it can be expressed as a function of differences of true skills. The ratings $r_t$ are, however, constructed in the same way as in \eqref{eq:sdUpdate}---that is, without knowledge of the true skills $s_t$.

Note that this assumption is quite general. For example, in the simple case of a win/loss game between two players, $A$ and $B$, at time $t$, the probabilities of the game outcomes are determined by $s_t^{(A)} - s_t^{(B)}$, which can take an arbitrary value, making the model fully general. The specific form of $f(y_t \mid s_t)$ is therefore not restrictive but allows for a comparison between $r_t$ and $s_t$.

We show that even though $r_{t+1}^{(i)}$ is constructed only from $r_t$ and $y_t$ (through $\nabla_i(r_t; y_t)$), it is related to $s_t$. Specifically, we demonstrate in the following proposition that when $r_t^{(i)} > s_t^{(i)}$ and $r_t^{(j)} = s_t^{(j)}$ for all $j \neq i$, the expected value $\mathrm{E} [ \nabla_i(r_t; y_t) \mid s_t ]$ is negative. In other words, when the rating of player $i$ is overestimated, the expected update of their rating is negative. Conversely, when $r_t^{(i)} < s_t^{(i)}$ and $r_t^{(j)} = s_t^{(j)}$ for all $j \neq i$, $\mathrm{E} [ \nabla_i(r_t; y_t) \mid s_t ]$ is positive. Finally, when $r_t^{(j)} = s_t^{(j)}$ for all $j$, $\mathrm{E} [ \nabla_i(r_t; y_t) \mid s_t ]$ equals zero, which is consistent with Proposition \ref{prop:expected}. We call this the \emph{reversion property}. The top plot in Figure \ref{fig:reverting} illustrates this behavior---the rating tends to oscillate around the true skill level, and when the skill level changes, the rating tends to follow. The score-driven rating system thus produces meaningful ratings even in the realistic case where the true skills of players are unobserved and differ from their ratings.

\begin{proposition}[Reversion Dynamics]
\label{prop:reversion}
Let $f(y_t \mid s_t)$ be a twice differentiable function with support independent of $s_t$. Furthermore, assume that $f(y_t \mid s_t)$ is strictly log-concave. Let $r_t^{(j)} = s_t^{(j)}$ for all $j \neq i$. Then, $\mathrm{E} [ \nabla_i(r_t; y_t) \mid s_t ]$ has the same sign as $s_t^{(i)} - r_t^{(i)}$.
\end{proposition}

\begin{proof}
First, let us consider the case of $s_t^{(i)} - r_t^{(i)} = 0$. As $s_t^{(i)} = r_t^{(i)}$, it follows directly from Proposition \ref{prop:expected} that
$$
\begin{aligned}
\mathrm{E} \left[ \nabla_i(r_t; y_t) \mid s_t \right] &= \mathrm{E} \left[ \nabla_i(s_t; y_t) \mid s_t \right] \\
&= 0.
\end{aligned}
$$
Next, let us consider the case of $s_t^{(i)} - r_t^{(i)} > 0$. From Proposition \ref{prop:decrease} and $s_t^{(i)} > r_t^{(i)}$, together with Proposition \ref{prop:expected}, it follows that
$$
\begin{aligned}
\mathrm{E} \left[ \nabla_i(r_t; y_t) \mid s_t \right] &= \int_{\mathcal{Y}} \nabla_i(r_t; y_t) f(y_t \mid s_t) d y_t \\
&> \int_{\mathcal{Y}} \nabla_i(s_t; y_t) f(y_t \mid s_t) d y_t \\
&= \mathrm{E} \left[ \nabla_i(s_t; y_t) \mid s_t \right] \\
&= 0. \\
\end{aligned}
$$
Finally, the case of $s_t^{(i)} - r_t^{(i)} < 0$ is analogous, with the inequality reversed.
\end{proof}

\begin{remark}
From Proposition \ref{prop:sum}, it follows that $\frac{1}{n} \sum_{i=1}^n r_t^{(i)} = r_{\text{init}}$ for all $t$. However, $s_t^{(i)}$, $i=1,\ldots,n$, can have arbitrary values in general and $\frac{1}{n} \sum_{i=1}^n s_t^{(i)}$ can vary over time. This allows us to consider the case $\frac{1}{n} \sum_{i=1}^n r_t^{(i)} \neq \frac{1}{n} \sum_{i=1}^n s_t^{(i)}$ in Proposition \ref{prop:reversion}.
\end{remark}

\begin{remark}
It is important to acknowledge that our reversion property does not imply mean reversion. Consider the simple case of a constant true skill for player $i$ over time, i.e., $s_t^{(i)} = s$ for all $t$. It is not implied that the long-term average of $r_t^{(i)}$ necessarily equals $s$. The bottom plot in Figure \ref{fig:reverting} illustrates an example where $r_t^{(i)}$ appears to have a long-term average, but one that differs from $s$. The ``convergence'' behavior is influenced by various factors, including the mechanism used to pair players for games (or, more generally, to select the sets of players for games). We leave this issue for future research.
\end{remark}

\begin{figure}
\centering
\includegraphics[scale=0.55]{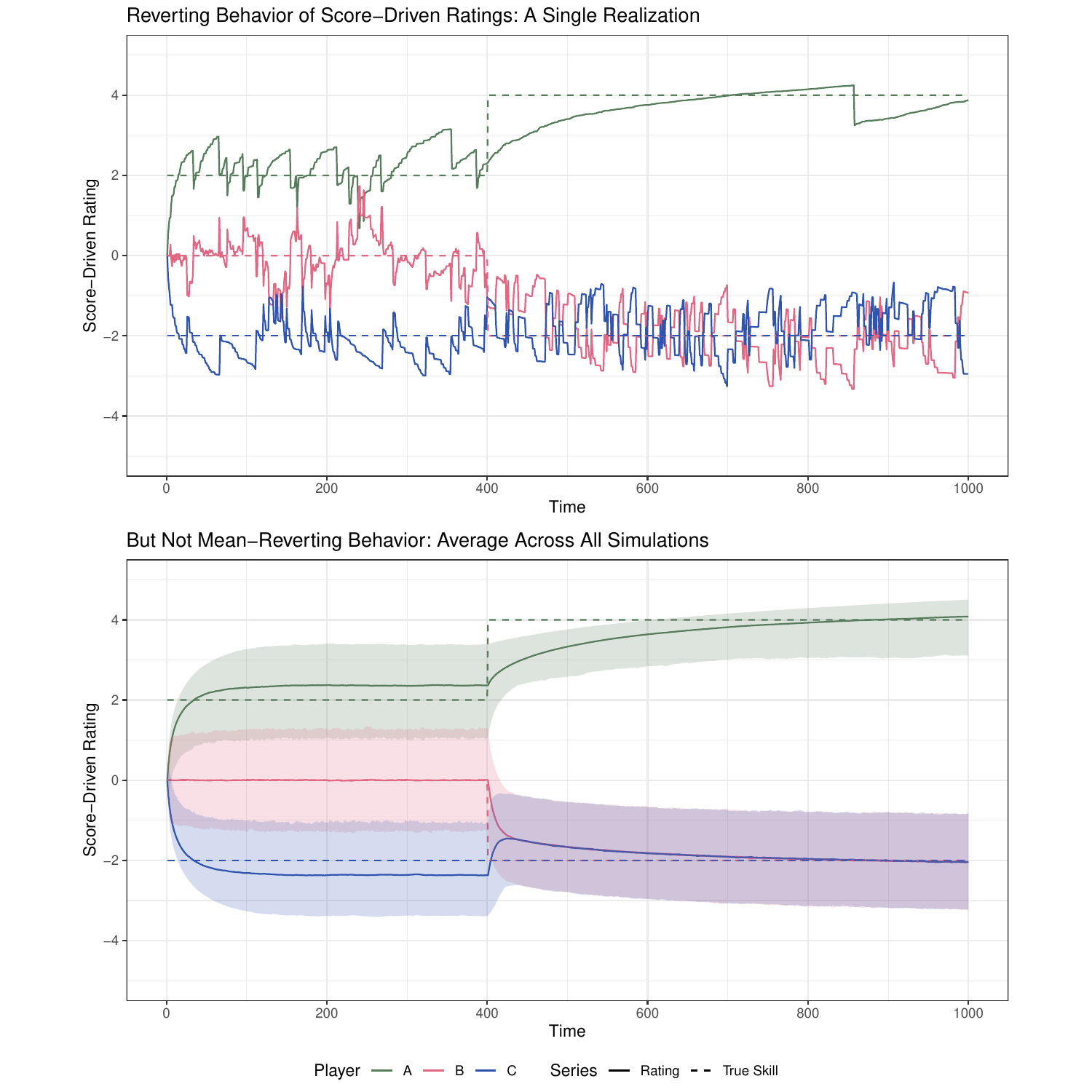}
\caption{Simulated paths of score-driven ratings for three players. Players are paired randomly for win/loss games. The top plot shows a single realization, while the bottom plot presents the average across all simulations with 95 percent confidence bands.}
\label{fig:reverting}
\end{figure}

\section{Conclusion}
\label{sec:con}

We deal with the rating of players and teams and present a generalization of the Elo rating system that utilizes the score of the probability distribution of game outcomes. We show that using the score leads to several desirable properties, making it suitable as an updating term for the rating after a game is played. These results are of interest to statisticians and sports analysts. The contribution of the paper is twofold:
\begin{enumerate}
\item Our results provide a theoretical rationale for many existing time-series models of sports performance, such as those employed in \cite{Gorgi2019}, \cite{Koopman2019}, \cite{Lasek2021}, \cite{Holy2022f}, and \cite{Holy2025a}.
\item Our approach can also serve as a guide for developing new models. First, a stochastic model for the game outcome needs to be specified. Its probability mass or density function must be a strictly log-concave function that can be expressed in terms of differences in ratings. Then, the score with respect to the ratings can be utilized as the updating term for the ratings.
\end{enumerate}

The properties presented here---namely, the equivalence to the Elo rating system for win/loss game outcomes under the logistic link, the zero expected value of the score, the summation of scores to zero over all players, and the decreasing score with increasing player rating---are straightforward to prove, highlighting the simplicity and elegance of the proposed approach. The main challenge lies in establishing the convergence behavior of the ratings toward the unobserved underlying true skills. We show that the score-driven rating system exhibits a reversion property, which validates the score-driven ratings in the general case of unknown true skills to a certain degree. Deriving stronger convergence results, however, remains a topic for future research. We believe that pursuing such results would require imposing restrictive assumptions and/or focusing on particular cases of the score-driven rating system.


%


\end{document}